\documentclass{article}

     \usepackage[final,nonatbib]{neurips_2018}

\usepackage[utf8]{inputenc} 
\usepackage[T1]{fontenc}    
\usepackage{hyperref}       
\usepackage{url}            
\usepackage{booktabs}       
\usepackage{amsfonts}       
\usepackage{nicefrac}       
\usepackage{microtype}      

\usepackage{amsmath,amsthm}
\title{Fairness Through Computationally-Bounded Awareness}

\author{
  Michael P. Kim\thanks{Supported by NSF Grant CCF-1763299.}\\
  Stanford University\\
  \texttt{mpk@cs.stanford.edu}
  \And
  Omer Reingold\footnotemark[1]\\
  Stanford University\\
  \texttt{reingold@stanford.edu}
  \And 
  Guy N. Rothblum\thanks{Supported by ISF grant No. 5219/17.}\\
  Weizmann Institute of Science\\
  \texttt{rothblum@alum.mit.edu}
}

\newcommand{\A}{{\cal A}}
\newcommand{\C}{{\cal C}}
\newcommand{\D}{{\cal D}}
\newcommand{\F}{{\cal F}}

\renewcommand{\H}{{\cal H}}
\newcommand{\M}{{\cal M}}
\newcommand{\N}{\mathbb{N}}
\newcommand{\R}{\mathbb{R}}
\newcommand{\U}{{\cal U}}

\newcommand{\X}{{\cal X}}

\newcommand{\poly}{\textnormal{poly}}

\newcommand{\eps}{\varepsilon}
\newcommand{\sgn}{\textnormal{\textbf{sgn}}}

\newcommand{\card}[1] {\left\vert #1 \right\vert}
\newcommand{\set}[1] {\left\{ #1 \right\}}
\newcommand{\given}{\big\vert}

\DeclareMathOperator*{\E}{\textnormal{\textbf{E}}}

\newcommand{\norm}[1] {\left\| #1 \right\|}
\newcommand{\grad}{\nabla}

\newtheorem{theorem}{Theorem}
\newtheorem*{theorem*}{Theorem}
\newtheorem{proposition}[theorem]{Proposition}
\newtheorem*{proposition*}{Proposition}
\newtheorem*{definition}{Definition}
\newtheorem{lemma}[theorem]{Lemma}

\newcounter{algorithm}

\begin{document}

\maketitle

\setcounter{footnote}{0}

\begin{abstract}
We study the problem of fair classification within the versatile framework of Dwork~\emph{et al.} \cite{fta}, which assumes the existence of a metric that measures similarity between pairs of individuals.
Unlike earlier work, we do not assume that the entire metric is known to the learning algorithm; instead, the learner can query this \emph{arbitrary} metric a bounded number of times.
We propose a new notion of fairness called \emph{metric multifairness} and show how to achieve this notion in our setting.
Metric multifairness is parameterized by a similarity metric $d$ on pairs of individuals to classify and a rich collection $\C$ of (possibly overlapping) ``comparison sets" over pairs of individuals.  At a high level,
metric multifairness guarantees that {\em similar subpopulations are treated similarly}, as long as these subpopulations are identified within the class $\C$.
\end{abstract}

\section{Introduction}
More and more, machine learning systems are being used to make
predictions about people.  Algorithmic predictions are now being
used to answer questions of significant personal consequence;
for instance,
\emph{Is this person likely to repay a loan?} \cite{creditscores} or
\emph{Is this person likely to recommit a crime?} \cite{propublica}.
As these classification systems have become more ubiquitous,
concerns have also grown that classifiers obtained via machine learning might discriminate based on sensitive attributes like race, gender, or sexual orientation.
Indeed, machine-learned classifiers run the risk of perpetuating or amplifying historical biases present in the training data.
Examples of discrimination in classification have been well-illustrated
\cite{creditscores,propublica,corbett2017algorithmic,kleinberg2016inherent,hardt2016equality,gendershades}; nevertheless, developing a systematic approach to
fairness has been challenging.  Often, it feels that the objectives
of fair classification are at odds with obtaining high-utility predictions.

In an influential work, Dwork~\emph{et al.} \cite{fta} proposed a framework to resolve the
apparent conflict between utility and fairness,
which they call ``fairness through awareness."
This framework takes the perspective that a fair classifier should
\emph{treat similar individuals similarly}.  The work formalizes
this abstract goal by assuming
access to a task-specific similarity
metric $d$ on pairs of individuals.
The proposed notion of
fairness requires that if the distance between two individuals is small,
then the predictions of a fair classifier cannot be very different.
More formally, for some small constant $\tau \ge 0$,
we say a hypothesis
$f:\X \to [-1,1]$ satisfies
\emph{$(d,\tau)$-metric fairness}\footnote{Note the
definition given in \cite{fta} is slightly different;
in particular, they propose a more general
Lipschitz condition, but fix $\tau = 0$.}
if the following (approximate)
Lipschitz condition holds for all pairs of
individuals from the population $\X$.
\begin{equation}
\forall x,x' \in \X \times \X:~~ \card{f(x) - f(x')} \le d(x,x') + \tau
\end{equation}
Subject to these intuitive similarity constraints, the classifier may be
chosen to maximize utility.
Note that, in general, the metric may be designed externally
(say, by a regulatory agency) to address legal and ethical concerns,
independent from the task of learning.
In particular, in certain settings, the metric designers may have access to a different set of features than the learner.
For instance, perhaps the metric designers have access to sensitive attributes, but for legal, social, or pragmatic reasons, the learner does not.
In addition to its conceptual simplicity,
the modularity of fairness through awareness makes it a
very appealing framework.
Currently, there are many (sometimes contradictory)
notions of what it means for a classifier to be fair \cite{kleinberg2016inherent,corbett2017algorithmic,feller2016computer,hardt2016equality,multi},
and there is much debate on which definitions should be applied in
a given context. Discrimination comes in many forms and classification is
used in a variety of settings, so naturally, it is hard to imagine any
universally-applicable definition of fairness. Basing fairness on a similarity metric offers a flexible approach for formalizing a variety of guarantees and protections from discrimination.

Still, a challenging aspect of this approach is the assumption that the
similarity metric is known for all pairs of individuals.\footnote{Indeed,
\cite{fta} identifies this assumption as
``one of the most challenging aspects'' of the framework.}
Deciding on an appropriate metric is
itself a delicate matter and could require human input from sociologists,
legal scholars, and specialists with domain expertise.
For instance, in the loan repayment example, a simple, seemingly-objective
metric might be a comparison of credit scores. A potential concern, however,
is that these scores might themselves be biased (i.e.~encode historical
discriminations). In this case, a more nuanced metric requiring human
input may be necessary.
Further, if the metric depends on features that are latent to the learner (e.g.~some missing sensitive feature) then the metric could appear \emph{arbitrarily complex} to the learner.
As such, in many realistic settings,
the resulting metric will not
be a simple function of the learner's feature
vectors of individuals.

In most machine learning applications, where the universe of individuals
is assumed to be very large, even writing down
an appropriate metric could be completely infeasible.  In these cases,
rather than require the metric value to be specified
for all pairs of individuals, we could instead ask
a panel of experts to provide similarity scores for a
\emph{small sample} of pairs of individuals.
While it is information-theoretically impossible
to guarantee metric fairness from
a sampling-based approach,
we still might hope to provide
a strong, provable notion of fairness that maintains the
theoretical appeal and practical modularity of the
fairness through awareness framework.

In this work, we propose a new theoretical framework for fair classification
based on fairness through awareness -- which we dub ``fairness through
computationally-bounded awareness'' -- that eliminates the considerable
issue of requiring the metric to be known exactly.
Our approach maintains the simplicity and flexibility of fairness
through awareness, but provably only requires a small number of
random samples from the underlying metric, even
though we make no structural assumptions about the
metric.  In particular, our approach works even if
the metric provably cannot be learned.
Specifically, our notion will require that a fair classifier
\emph{treat similar subpopulations of individuals similarly},
in a sense we will make formal next.
While our definition relaxes fairness through awareness,
we argue that it still protects against important
forms of discrimination that the original work aimed to combat;
further, we show that stronger notions necessarily require a
larger sample complexity from the metric.  As in \cite{fta},
we investigate how to learn a classifier that achieves optimal
utility under similarity-based fairness constraints,
assuming a weaker model of limited access to the metric.
We give positive and negative results that show
connections between achieving our fairness notion and learning.

\section{Setting and metric multifairness}

\paragraph{Notation.}
Let $\X \subseteq \R^n$ denote the universe over individuals we wish to classify,
where $x \in \X$ encodes the features of an individual.
Let $\D$ denote the data distribution over individuals and labels supported on $\X \times \set{-1,1}$; denote by
$x,y \sim \D$ a random draw from this distribution.
Additionally, let $\M$ denote the metric sample
distribution over pairs of individuals.  For a subset
$S \subseteq \X \times \X$, we denote
by $(x,x') \sim S$ a random draw from the
distribution $\M$ conditioned on $(x,x') \in S$.
Let $d:\X\times \X\to[0,2]$ denote
the underlying fairness metric that maps pairs to
their associated distance.\footnote{In fact,
all of our results hold for a more general class
of non-negative symmetric distance functions.}

Our learning objective will be to minimize
the expectation over $\D$ of
some convex loss function $L:[-1,1]\times[-1,1] \to \R_+$
over a convex hypothesis class $\F$, subject
to the fairness constraints.
We focus on agnostically learning the hypothesis class of
linear functions with bounded weight; for some
constant $B > 0$, let $\F = [-B,B]^n$.
For $w \in \F$, define $f_w(x) = \langle w,x\rangle$, projecting $f_w(x)$ onto $[-1,1]$ to get a valid prediction.
We assume $\norm{x}_1 \le 1$ for all $x \in \X$;
this is without loss of generality, by normalizing
and increasing $B$ appropriately.

The focus on linear functions is not too
restrictive; in particular, by increasing the
dimension to $n' = O(n^k)$, we can learn any
degree-$k$ polynomial function of the original
features.  By increasing $k$, we can approximate
increasingly complex functions.

\subsection{Metric multifairness}
We define our relaxation of metric fairness
with respect to a rich class of statistical tests on the pairs of individuals.
Let a \emph{comparison} be any subset of the
pairs of $\X \times \X$.
Our definition, which we call \emph{metric multifairness}, is
parameterized by a collection of
comparisons $\C \subseteq 2^{\X \times \X}$
and requires that a hypothesis
appear Lipschitz according to all of the statistical tests defined by
the comparisons $S \in \C$.

\begin{definition}[Metric multifairness]
Let $\C \subseteq 2^{\X \times \X}$ be a collection of
comparisons and let $d:\X \times \X \to [0,2]$ be a metric.
For some constants $\tau \ge 0$, a hypothesis $f$ is
\emph{$(\C,d,\tau)$-metric multifair}
if for all $S \in \C$,
\begin{equation}
\E_{(x,x') \sim S}\big[\card{f(x)-f(x')}\big]
\le \E_{(x,x') \sim S}\big[d(x,x')\big] + \tau.
\end{equation}
\end{definition}
To begin, note that metric multifairness is indeed
a relaxation of metric fairness; if we take the
collection $\C = \set{\set{(x,x')} : x,x' \in \X\times\X}$
to be the collection of all pairwise comparisons,
then $(\C,d,\tau)$-metric multifairness is
equivalent to $(d,\tau)$-metric fairness.

In order to achieve metric multifairness
from a small sample from the metric, however,
we need a lower bound on
the density of each comparison in $\C$; in particular,
we can't hope to enforce
metric fairness from a small sample.
For some $\gamma > 0$,
we say that a collection of comparisons $\C$ is \emph{$\gamma$-large}
if for all $S \in \C$,
$\Pr_{(x,x') \sim \M}[(x,x') \in S] \ge \gamma$.
A natural next choice for $\C$ would be a collection
of comparisons that represent the Cartesian products
between traditionally-protected groups, defined by
race, gender, etc.  In this case, as long as the
minority populations are not too small, then a random
sample from the metric will concentrate around the
true expectation, and we could hope to enforce
this statistical relaxation of metric fairness.
While this approach is information-theoretically
feasible, its protections are very weak.

To highlight this weakness, suppose we want to predict the
probability individuals will repay a loan, and our metric
is an adjusted credit score.  Even after adjusting scores,
two populations $P,Q \subseteq \X$ (say, defined by race)
may have large average distance because \emph{overall} $P$
has better credit than $Q$; still, within $P$ and $Q$,
there may be significant \emph{subpopulations}
$P' \subseteq P$ and $Q'\subseteq Q$ that should be
treated similarly (possibly representing the
qualified members of each group).
In this case, a coarse statistical relaxation
of metric fairness will not require that a classifier
treat $P'$ and $Q'$ similarly; instead, the classifier could
treat everyone in $P$ better than everyone in $Q$ --
including treating \emph{unqualified} members of $P$
better than \emph{qualified} members of $Q$.  Indeed,
the weaknesses of broad-strokes statistical
definitions served as motivation for the original
work of \cite{fta}.
We would like to choose a class $\C$ that
strengthens the fairness guarantees of metric
multifairness, but maintains its efficient sample
complexity.

\paragraph{Computationally-bounded awareness.}
While we can define metric multifairness with respect
to any collection $\C$,
typically, we will think of $\C$ as a rich class of
overlapping subsets; equivalently, we can think of
the collection $\C$ as
an expressive class of boolean functions, where for $S \in \C$,
$c_S(x,x') = 1$ if and only if $(x,x') \in S$.
In particular, $\C$ should be much more expressive
than simply defining comparisons across
traditionally-protected groups.
The motivation for choosing such an expressive class $\C$
is exemplified in the following proposition.
\begin{proposition}
\label{prop:avg}
Suppose there is some $S \in \C$, such that
$\E_{(x,x') \sim S}[d(x,x')] \le \eps$.
Then if $f$ is $(\C,d,\tau)$-metric multifair,
then $f$ satisfies $(d,(\eps+\tau)/p)$-metric fairness
for at least a $(1-p)$-fraction of the pairs in $S$.
\end{proposition}
That is, if there is some subset $S \in \C$
that identifies a set of pairs whose metric distance
is small, then any metric multifair hypothesis
must also satisfy the stronger individual metric fairness notion
on many pairs from $S$.  This effect will compound if
many different (possibly overlapping) comparisons
are identified that have small average distance.
We emphasize that these small-distance
comparisons are not known before sampling from the
metric; indeed, this
would imply the metric was (approximately) known
\emph{a priori}.
Still, if the class $\C$ is rich enough to correlate
well with various comparisons that reveal
significant information about the metric,
then any metric multifair hypothesis will satisfy
\emph{individual-level} fairness on a significant
fraction of the population!

While increasing the expressiveness of $\C$ increases
the strength of the fairness guarantee, in order
to learn from a small sample,
we cannot choose $\C$ to be arbitrarily complex.
Thus, in choosing $\C$ we must balance the strength
of the fairness guarantee with the information
bottleneck in accessing $d$ through random
samples.  Our resolution to these competing
needs is complexity-theoretic: while
information-theoretically, we can't hope
to ensure fair treatment across
\emph{all} subpopulations, we can hope ensure fair
treatment across \emph{efficiently-identifiable}
subpopulations.
For instance, if we take $\C$ to be a family defined
according to some class of computations of bounded
dimension -- think, the set of conjunctions of a constant
number of boolean features or short decision trees --
then we can hope to accurately estimate and enforce
the metric multifairness conditions.
Taking such a bounded $\C$
ensures that a hypothesis will be fair on all comparisons
identifiable within this computational bound.
This is the sense in which metric multifairness
provides fairness through \emph{computationally-bounded} awareness.

\subsection{Learning model}

\paragraph{Metric access.}
Throughout, our goal is to learn a hypothesis from
noisy samples from the metric that satisfies multifairness.
Specifically, we assume an algorithm can obtain a small
number of independent random \emph{metric samples}
$(x,x',\Delta(x,x')) \in \X\times\X\times[0,2]$
where $(x,x') \sim \M$ is drawn at random over the
distribution of pairs of individuals, and $\Delta(x,x')$
is a random variable of bounded magnitude with
$\E[\Delta(x,x')] = d(x,x')$.

We emphasize that this is a very limited access model.
As Theorem~\ref{thm:SGD} shows we achieve
$(\C,d,\tau)$-metric multifairness from
a number of samples that depends \emph{logarithmically}
on the size of $\C$ independent of the complexity of the similarity metric.\footnote{Alternatively, for continuous
classes of $\C$, we can replace $\log(\card{\C})$ with
some notion of dimension (VC-dimension, metric entropy, etc.)
through a uniform convergence argument.}
Recall that $d:\X\times\X\to [0,2]$ can be an
\emph{arbitrary} symmetric function;
thus, the learner does not necessarily have
enough information to learn $d$.
Still, for exponentially-sized $\C$,
the learner can guarantee metric
multifairness from a polynomial-sized sample,
and the strength of the guarantee
will scale up with the complexity of $\C$
(as per Propsition~\ref{prop:avg}).

In order to ensure a strong notion of fairness,
we assume that the subpopulations we wish to
protect are well-represented in the pairs drawn
from $\M$.  This assumption, while important, is
not especially restrictive, as we think of the
metric samples as coming from a regulatory
committee or ethically-motivated party;
in other words, in practical settings, it is
reasonable to assume that one can choose the metric
sampling distribution based on the notion of
fairness one wishes to enforce.

\paragraph{Label access.}
When we learn linear families,
our goal will be to learn from a sample of labeled examples.
We assume the algorithm can ask for
independent random samples $x,y \sim \D$.

\paragraph{Measuring optimality.} To evaluate the utility guarantees of our learned predictions,
we take a comparative approach.
Suppose $\H \subseteq 2^{\X\times \X}$ is
a collection of comparisons.  For $\eps \ge 0$,
we say a hypothesis $f$ is \emph{$(\H,\eps)$-optimal} with respect to $\F$, if
\begin{equation}
\E_{x,y \sim \D}\left[L(f(x),y)\right]
\le \E_{x,y \sim \D}\left[L(f^*(x),y)\right] + \eps
\end{equation}
where $f^* \in \F$ is an optimal $(\H,d,0)$-metric multifair hypothesis.

\section{Learning a metric multifair hypothesis}
\label{sec:learning}
As in \cite{fta}, we formulate the problem of learning a fair set of predictions
as a convex program.
Our objective is to minimize the expected loss
$\E_{x,y \sim \D}[L(f(x),y)]$,
subject to the multifairness constraints defined
by $\C$.\footnote{For the sake of presentation,
throughout the theorem statements, we will assume that $L$ is $O(1)$-Lipschitz
on the domain of legal predictions/labels to guarantee bounded error; our
results are proved more generally.}
Specifically, we show that a simple variant of stochastic gradient
descent due to \cite{nesterov}
learns such linear families efficiently.

\begin{theorem}
\label{thm:SGD}
Suppose $\gamma,\tau,\delta > 0$ and $\C \subseteq 2^{\X \times \X}$ is
$\gamma$-large.
With probability at least $1-\delta$, stochastic switching subgradient
descent learns a hypothesis $w \in \F$ that is
$(\C,d,\tau)$-metric multifair and $(\C,O(\tau))$-optimal with
respect to $\F$ in $O\left(\frac{B^2n^2\log(n/\delta)}{\tau^2}\right)$
iterations from $m = \tilde{O}\left(\frac{\log(\card{\C}/\delta)}{\gamma\tau^2}\right)$
metric samples.  Each iteration uses at most $1$ labeled
example and can be implemented in
$\tilde{O}\left(\card{C} \cdot n \cdot \poly(1/\gamma,1/\tau,\log(1/\delta)) \right)$
time.
\end{theorem}
Note that the metric sample complexity depends
logarithmically on $\card{\C}$.  Thus, information-theoretically,
we can hope to enforce metric mutlifairness with a class
$\C$ that grows exponentially and still be efficient.
While the running time of each iteration depends on $\card{\C}$, note that the number of iterations
is \emph{independent} of $\card{\C}$.  In Section~\ref{sec:agnostic}, we show conditions on
$\card{\C}$ under which we can speed up the
running time of each iteration to depend logarithmically
on $\card{\C}$.

\begin{figure}[b!]
{\refstepcounter{algorithm} \label{alg:switching}{\bf Algorithm~\thealgorithm:} Switching Subgradient Descent}

\fbox{\parbox{\textwidth}{
Let $\tau > 0$, $T \in \N$, and $\C \subseteq 2^{\X \times \X}$.\\
Initialize $w_0 \in \F = [-B,B]^n; W = \emptyset$\\
For $k = 1,\hdots, T$:
\begin{itemize}
\item If $\exists S \in \C$ such that $\hat{R}_S(w_k) > 4\tau/5$: \hfill
\texttt{// some constraint violated}
\begin{itemize}
  \item $S_k \gets$ any $S \in \C$ such that $\hat{R}_S(w_k) > 4\tau/5$
  \item $w_{k+1} \gets w_k - \frac{\tau}{M^2} \grad R_{S_k}(w_k)$ \hfill \texttt{/* step according to constraint}\\ \null\hfill \texttt{ project onto F if necessary */}
\end{itemize}
\item Else: \hfill \texttt{// no violations found}
\begin{itemize}
  \item $W \gets W \cup \set{w_k}$\hfill \texttt{// update set of feasible iterates}
  \item $w_{k+1} \gets w_k - \frac{\tau}{GM} \grad L(w_k)$ \hfill \texttt{/* step according to objective}\\ \null\hfill \texttt{ project onto F if necessary */}
\end{itemize}
\end{itemize}
Output $\bar{w} = \frac{1}{\card{W}} \cdot \sum_{w \in W} w$
\hfill\texttt{// output average of feasible iterates}
}
}
\end{figure}

We give a description of the switching
subgradient method in Algorithm~\ref{alg:switching}.
At a high level, at each iteration, the procedure checks to
see if any constraint is significantly violated.  If it finds
a violation, it takes a (stochastic) step towards feasibility.
Otherwise, it steps according a stochastic subgradient for the objective.

For convenience of analysis, we define the residual
on the constraint defined by $S$ as follows.
\begin{equation}
R_{S}(w) = \E_{(x,x') \sim S}\big[\card{f_w(x) - f_w(x')}\big]
- \E_{(x,x') \sim S}\big[d(x,x') \big]
\end{equation}
Note that $R_S(w)$ is convex in the predctions $f_w(x)$ and
thus, for linear families is convex in $w$.
We describe the algorithm assuming access to the
following estimators, which we can implement
efficiently (in terms of time and samples).
First, we assume we can estimate the residual $\hat{R}_S(w)$ on each $S \in \C$
with tolerance $\tau$ such that
for all $w \in \F$,
$\card{R_S(w) - \hat{R}_S(w)} \le \tau$.
Next, we assume access to a stochastic subgradient oracle for the constraints
and the objective.
For a function $\phi(w)$, let $\partial \phi(w)$ denote the set of
subgradients of $\phi$ at $w$.  We abuse notation, and let $\grad \phi(w)$ refer to a vector-valued random variable
where $\E[\grad \phi(w) \given w] \in \partial \phi(w)$.
We assume access to stochastic subgradients
for $\partial R_S(w)$
for all $S \in \C$ and $\partial L(w)$.
We include a full analysis of the algorithm and
proof of Theorem~\ref{thm:SGD} in
the Appendix.

\subsection{Post-processing for metric multifairness}

\label{sec:post}
One specific application of this result is as a way post-process
learned predictions to ensure fairness.
In particular, suppose we are given the predictions
from some pre-trained model for $N$ individuals, but
are concerned that these predictions may not be fair.
We can use Algorithm~\ref{alg:switching} to
post-process these labels to select near-optimal metric
multifair predictions.  Note these
predictions will be optimal with respect to
the \emph{unconstrained} family of predictions
-- not just predictions that come from
a specific hypothesis class (like linear functions).

Specifically, in this setting
we can represent an unconstrained set of predictions
as a linear hypothesis in $N$ dimensions:  take $B = 1$,
and let the feature vector for $x_i \in \X$ be the $i$th standard
basis vector.
Then, we can think of the input labels to Algorithm~\ref{alg:switching}
to be the \emph{output} of
any predictor that was trained separately.\footnote{Nothing in our analysis required labels $y \in \set{-1,1}$; we can instead take the labels $y \in [-1,1]$.}  For instance, if
we have learned a highly-accurate model, but are unsure of its
fairness, we can instantiate our framework with, say, the squared loss
between the original predictions and the returned predictions; then, we can
view the program as a procedure to \emph{project} the highly-accurate predictions onto
the set of metric multifair predictions. Importantly, our procedure
only needs a small set of samples from the metric and \emph{not}
the original data used to train the model.

Post-processing prediction models for fairness
has been studied in a few contexts \cite{woodworth2017learning,multi,kgz}.
This post-processing setting should be
contrasted to these settings.  In our setting,
the predictions are not required to generalize
out of sample (in terms of loss or fairness).
On the one hand, this means the metric
multifairness guarantee does not generalize
outside the $N$ individuals;
on the other hand, because the predictions
need not come from a bounded hypothesis class, 
their utility can only
improve compared to learning a metric multifair
hypothesis directly.

In addition to preserving the utility of previously-trained classifiers,
separating the tasks of training for utility and enforcing fairness
may be desirable when intentional malicious discrimination may be anticipated.
For instance, when addressing the forms of
racial profiling that can occur through targeted
advertising (as described in \cite{fta}), we may
not expect self-interested advertisers to
adhere to classification under strict fairness
constraints, but it stands to reason that prominent advertising platforms
might want to prevent such blatant abuses of their platform.  In this setting, the platform could
impose metric multifairness after the advertisers
specify their ideal policy.

\section{Reducing search to agnostic learning}
\label{sec:agnostic}
As presented above,
the switching subgradient descent method converges to a
nearly-optimal point in a bounded number of subgradient steps,
independent of $\card{\C}$.  The catch is that at the beginning
of each iteration, we need to search over $\C$ to determine
if there is a significantly violated multifairness constraint.
As we generally want to take $\C$ to
be a rich class of comparisons, in many cases $\card{\C}$ will be
prohibitive.  As such, we would hope to find violated constraints
in sublinear time, preferably even poly-logarithmic in $\card{C}$.
Indeed, we show that if a concept class $\C$ admits an efficient
agnostic learner, then we can solve the violated constraint search problem
over the corresponding collection of comparisons efficiently.

Agnostic learning can be phrased as a problem of detecting correlations.
Suppose $g,h: \U \to [-1,1]$, and let $\D$ be some distribution supported
on $\U$.  We denote the correlation between $g$ and $h$ on $\D$ as
$\langle g,h \rangle = \E_{i \sim \D}[g(i)\cdot h(i)]$.
We let $\C \subseteq [-1,1]^\U$ denote the \emph{concept class}
and $\H \subseteq [-1,1]^\U$ denote the \emph{hypothesis class}.
The task of agnostic learning can be stated as follows: given sample access
over some distribution $(i,g(i)) \sim \D\times [-1,1]$
to some function $g \in [-1,1]^N$, find some hypothesis
$h \in \H$ that is comparably correlated with $g$ as the best $c \in \C$.
That is, given access to $g$, an agnostic learner with accuracy $\eps$
for concept class $\C$ returns some $h$ from the hypothesis class $\H$
such that
\begin{equation}
\langle g,h \rangle + \eps \ge \max_{c \in \C}\langle g,c \rangle.
\end{equation}
An agnostic learner is
typically considered efficient if it runs in polynomial time in
$\log(\card{\C})$ (or an appropriate notion of dimension of $\C$),
$1/\eps$, and $\log(1/d)$.  Additionally, distribution-specific
learners and learners with query access to the function have been studied
\cite{gopalan2008agnostically,feldman2010distribution}.
In particular, membership queries tend to make agnostic learning easier.
Our reduction does not use any metric samples other than those that the
agnostic learner requests.  Thus, if we are able to query a panel
of experts according to the learner, rather than randomly, then
an agnostic learner that uses queries could be used to speed up
our learning procedure.

\begin{theorem}
Suppose there is an algorithm $\A$ for agnostic learning
the concept class $\C$ with hypothesis class $\H$
that achieves accuracy $\eps$ with probability
$1-\delta$ in time $T_\A(\eps,\delta)$ from $m_\A(\eps,\delta)$ labeled samples.
Suppose that $\C$ is $\gamma$-large.
Then, there is an algorithm that, given access to $T = \tilde{O}\left(
\frac{B^2n^2}{\tau^2}\right)$ labeled
examples,
outputs a set of predictions that are $(\C,d,\tau)$-metric multifair 
and $(\H,O(\tau))$-optimal with respect to $\F = [-B,B]^n$
that runs in time
$\tilde{O}\left(\frac{T_\A(\gamma\tau,\delta/T) \cdot B^2n^2}{\gamma^2\tau^2}\right)$
and requires
$m = \tilde{O}\left(\frac{\log(\card{\C})}{\gamma \tau^2} +
\frac{n_\A(\gamma\tau,\delta/T)}{\gamma\tau^2}\right)$ metric samples.
\end{theorem}

When we solve the convex program with switching subgradient
descent, at the beginning of each iteration,
we check if there is any $S \in \C$ such that the residual quantity
$R_S(w)$ is greater than some threshold.
If we find such an $S$,
we step according to the subgradient of $R_S(w)$.
In fact, the proof of the convergence of switching
subgradient descent reveals that
as long as when there is some $S \in \C$ where
$R_S(w)$ is in violation, we can find some
$R_{S'}(w) > \rho$ for some constant $\rho$,
where $S' \in \H \subseteq [-1,1]^{\X\times\X}$,
then we can argue that
the learned hypothesis will be $(\C,d,\tau)$-metric
multifair and achieve utility commensurate with the
best $(\H,d,0)$-metric multifair hypothesis.

We show a general reduction from the problem of searching
for a violated comparison $S \in \C$ to the problem of agnostic
learning over the corresponding family of boolean functions.
In particular, recall that for a collection of
comparisons $\C \subseteq 2^{\X\times \X}$,
we can also think of $\C$ as a family of boolean concepts,
where for each $S \in \C$, there is an associated boolean
function $c_S:\X\times\X \to \set{-1,1}$ where $c_S(x_i,x_j) = 1$ if and only
if $(x_i,x_j) \in S$.
We frame this search problem as an agnostic learning problem, where
we design a set of ``labels'' for each pair $(x,x') \in \X\times\X$
such that any function that is highly correlated with these labels
encodes a way to update the parameters towards feasibility.
\begin{proof}
Recall the search problem: given a current hypothesis $f_w$, is
there some $S \in \C$ such that
$R_S(w) = \E_{(x,x')\sim S}[\card{f_w(x_i) - f_w(x_j)} - d(x_i,x_j)] > \tau$?
Consider the labeling each pair $(x_i,x_j)$ with
$v(x_i,x_j) = \card{f_w(x_i) - f_w(x_j)} - d(x_i,x_j)$.
Let $\rho = R_{\X\times \X}(w)$; note that
we can treat $\rho$ as a constant for all $S \in \C$.
Further, suppose $S \in \C$ is such that $R_S(w) > \tau$, or
equivalently, $\E_{(x_i,x_j) \sim S}[v(x_i,x_j)] > \tau$.
Then, by the assumption that $\C$ is $\gamma$-large, the correlation
between the corresponding boolean function $c_S$ and labels $v$ can be lower 
bounded as $\langle c_S,v \rangle > 2\gamma \tau - \rho$.
Suppose we have an agnostic learner that returns a hypothesis
$h$ with $\eps < \gamma \tau$ accuracy.  Then, we know that
$\langle h, v \rangle \ge \gamma \tau - \rho$
by the lower bound on the optimal $c_S \in \C$.
Then, consider the function $R_{h}(w)$
defined as follows.
\begin{align}
R_{S_h}(w) &= \E_{(x,x') \sim \X\times\X}\left[\left(\frac{h(x,x')+1}{2}\right)\cdot
v(x,x')\right]\\
&= \frac{\langle h,v \rangle +\rho}{2}
\end{align}
Thus, given that there exists some $S \in \C$ where $R_S(w) > \tau$,
we can find some real-valued comparison $S_h(x,x') = \frac{h(x,x')+1}{2}$,
such that $R_{S_h}(w) = \Omega(\langle h,v \rangle + \rho) \ge \Omega(\gamma\tau)$.
\end{proof}
Discovering a violation of at least $\Omega(\gamma\tau)$ guarantees
$\Omega(\gamma^2\tau^2)$ progress in the duality gap
at each step, so the theorem follows from the
analysis of Algorithm~\ref{alg:switching}.

\section{Hardness of learning metric multifair hypotheses}
\label{sec:lb}

In this section, we show that our algorithmic results
cannot be improved significantly.  In particular,
we focus on the post-processing setting of
Section~\ref{sec:post}.
We show that the metric sample complexity
is tight up to a $\Omega(\log\log(\card{\C}))$ factor
unconditionally.
We also show that some
learnability assumption on $\C$ is necessary in order
to achieve a high-utility $(\C,d,\tau)$-metric
multifair predictions efficiently.
In particular, we give a reduction from
inverting a boolean concept $c \in \C$
to learning a hypothesis $f$ that is
metric multifair on a collection $\H$ derived from $\C$,
where the metric samples
from $d$ encode information about the concept $c$.
Recall, that for any $\H$ and $d$, we can always output a trivial
$(\H,d,0)$-metric multifair hypothesis by outputting
a constant hypothesis.  This leads to a subtlety in our
reductions, where we need to leverage the learner's
ability to simultaneously satisfy the metric
multifairness constraints and achieve high utility.

Both lower bounds follow the same general construction.
Suppose we have some boolean concept class
$\C \subseteq \set{-1,1}^{\X_0}$ for some universe $\X_0$.
We will construct a new universe $\X = \X_0 \cup \X_1$
and define a collection of ``bipartite'' comparisons over subsets of
${\X_0 \times \X_1}$. Then, given samples from $(x_0,c(x_0))$,
we define corresponding metric values
where $d(x_0,x_1)$ is some function of $c(x_0)$ for
all $x_1 \in \X_1$.  Finally, we need to additionally
encode the objective of inverting $c$ into labels for
$x_0 \in \X_0$, such that to obtain good loss, the post-processor must invert $c$ on $\X_0$.
We give a full description of
the reduction in the Appendix.

\paragraph{Lower bounding the sample complexity.}

While we argued earlier that
\emph{some} relaxation of metric fairness is necessary if we want to
learn from a small set of metric samples, it is not clear
that multifairness with respect to $\C$ is the strongest relaxation
we can obtain.  In particular, we might hope to guarantee fairness on
\emph{all} large comparisons, rather than just a finite class $\C$.
The following theorem shows that such a hope is misplaced: in order
for an algorithm to guarantee that the Lipschitz condition holds in
expectation over a finite
collection of large comparisons $\C$, then either
the algorithm takes $\Omega(\log\card{\C})$
random metric samples, or the algorithm
outputs a set of nearly useless predictions.
For concreteness, we state the theorem in the
post-processing setting of Section~\ref{sec:post};
the construction can be made to work in the learning
setting as well.
\begin{theorem}
\label{thm:lb}
Let $\gamma,\tau > 0$ be constants and
suppose $\A$ is an algorithm that
has random sample access to $d$ and outputs
a $(\C,d,\tau)$-metric multifair set of predictions
for $\gamma$-large $\C$.
Then, $\A$ takes $\Omega(\log\card{\C})$
random samples from $d$ or outputs a set of predictions
with loss that approaches the loss achievable with no metric
queries.
\end{theorem}
The construction uses a reduction from the problem
of learning a linear function;
we then appeal to a
lower bound from linear algebra on the number of random queries
needed to span a basis.

\paragraph{Hardness from pseudorandom functions.}

Our reduction implies that a post-processing algorithm
for $(\C,d,\tau)$-metric multifairness with
respect to an arbitrary metric $d$ gives us
a way of distinguishing functions in $\C$ from
random.
\begin{proposition}[Informal]
\label{prop:lb}
Assuming one-way functions exist, there is no
efficient algorithm for computing $(\C,\tau)$-optimal
$(\C,d,\tau)$-metric multifair predictions
for general $\C,d,$ and constant $\tau$.
\end{proposition}
Essentially, without assumptions
that $\C$ is a learnable class of boolean functions, some
nontrivial running time dependence on $\card{\C}$ is necessary.
The connection between learning and
pseudorandom functions \cite{valiant1984theory,ggm}
is well-established;
under stronger cryptographic assumptions as in \cite{prfs},
the reduction implies that a running time of
$\Omega(\card{\C}^{\alpha})$ is necessary for
some constant $\alpha > 0$.

\section{Related works and discussion}
Many exciting recent works have investigated fairness in
machine learning.  In particular, there is much debate on
the very definitions of what it means for a classifier to be fair
\cite{kleinberg2016inherent, chouldechova2017fair, pleiss2017fairness, hardt2016equality, corbett2017algorithmic,multi}.
Beyond the work of Dwork~\emph{et al.} \cite{fta}, our work bears most similarity
to two recent works of H\'{e}bert-Johnson~\emph{et al.}
and Kearns~\emph{et al.} \cite{multi,kearns2017preventing}.
As in this work, both of these papers investigate notions of fairness that aim to
strengthen the guarantees of statistical notions, while maintaining
their practicality.  These works also both draw connections
between achieving notions of fairness and efficient agnostic learning.
In general, agnostic learning is considered a notoriously hard computational
problem \cite{kearns,agnostic,feldman2012agnostic};
that said, in the context of fairness in machine learning,
\cite{kearns2017preventing} show that using heuristic methods to
agnostically learn linear hypotheses seems to work well in practice.

Metric multifairness does not directly generalize either \cite{multi} or \cite{kearns2017preventing}, but we argue that it provides a more flexible alternative to these approaches for subpopulation fairness.  In particular, these works aim to achieve specific notions of fairness -- either calibration or equalized error rates -- across a rich class of subpopulations.  As has been well-documented \cite{kleinberg2016inherent,chouldechova2017fair,pleiss2017fairness}, calibration and equalized error rates, in general, cannot be simultaneously satisfied.
Often, researchers frame this incompatibility as a choice: either you satisfy calibration or you satisfy equalized error rates; nevertheless, there are many applications where some interpolation between accuracy (\`a la calibration) and corrective treatment (\`a la equalized error rates) seems appropriate.

Metric-based fairness offers a way to balance these conflicting fairness desiderata. In particular, one could design a similarity metric that preserves accuracy in predictions and separately a metric that performs corrective treatment, and then enforce metric multifairness on an appropriate combination of the metrics.  For instance, returning to the loan repayment example, an ideal metric might be a combination of credit scores (which tend to be calibrated) and a metric that aims to increase the loans given to historically underrepresented populations (by, say, requiring the top percentiles of each subpopulation be treated similarly).
Different combinations of the two metrics
would place different weights on the degree of calibration and corrective discrimination in the resulting predictor.
Of course, one could equally apply this metric in the framework of \cite{fta}, but the big advantage with
metric multifairness is that we only need a small sample from the metric to provide a relaxed, but still strong guarantee of fairness.

We are optimistic that metric multifairness will provide an avenue towards implementing metric-based fairness notions.
At present, the results are theoretical, but we
hope this work can open the door to empirical studies
across diverse domains, especially since one of the
strengths of the framework is its generality.
We view testing the empirical performance of
metric multifairness with various choices of metric
$d$ and collection $\C$ as an exciting direction for
future research.

Finally, two recent theoretical works also investigate extensions to the fairness through awareness framework of \cite{fta}. Gillen~{\em et al.} \cite{GillenJKR18} study metric-based individually fair online decision-making in the presence of an unknown fairness metric.
In their setting, every day, a decision maker must
choose between candidates available on that day;
the goal is to have the decision maker's choices
appear metric fair on each day (but not across days).
Their work makes a strong learnability assumption
about the underlying metric; in particular, they
assume that the unknown metric is a Mahalanobis metric, whereas our focus is on fair classification when the metric is unknown and unrestricted.
Rothblum and Yona \cite{RothlbumY18} study fair machine learning under a different relaxation of metric fairness, which they call {\em approximate} metric fairness. They assume that the metric is fully specified and known to the learning algorithm, whereas our focus is on addressing the challenge of an unknown metric. Their notion of approximate metric fairness aims to protect {\em all} (large enough) groups, and thus, is more strict than metric multifairness.

\paragraph{Acknowledgements.}
\emph{The authors thank Cynthia Dwork, Roy Frostig, Fereshte Khani, Vatsal Sharan, Paris Siminelakis, and Gregory Valiant for helpful conversations and feedback on earlier drafts of this work.  We thank the anonymous reviewers for their careful reading and suggestions on how to improve the clarity of the presentation.}

\bibliographystyle{plain}
\bibliography{refs}

\appendix

\section{Analysis of Algorithm~\ref{alg:switching}}
\paragraph{Overview of analysis}
Here, we give a high-level overview of the analysis of
Algorithm~\ref{alg:switching}.  We defer some technical
lemmas to Appendix~\ref{app:sec:proof}.
We refer to $K_f$ as the set of ``feasible iterations'' where
we step according to the objective; that is,
\begin{equation}
K_f = \set{k \in [T] : \hat{R}_{S_k}(w_k) \le 4\tau/5}
\end{equation}

\paragraph{Fairness analysis}
We begin by showing that the hypothesis $\bar{w}$ that
Algorithm~\ref{alg:switching} returns satisfies metric multifairness.
\begin{lemma}
\label{lem:fairness}
Suppose for all $S \in \C$, the residual oracle $\hat{R}_S$
has tolerance $\tau/5$.  Then, $\bar{w}$ is $(\C,d,\tau)$-metric
multifair.
\end{lemma}
\begin{proof}
We choose our final hypothesis $\bar{w}$
to be the weighted average of the feasible iterates.
Note that the update rules for $K_f$ and $W$ 
imply that $\bar{w}$ is
a convex combination of hypotheses where no constraint appears
significantly violated,
$\bar{w} = \frac{1}{\card{K_f}}\cdot \sum_{k \in K_f} w_k$.
By convexity of $R_{S}$ we have the following inequality
for all $S \in \C$.
\begin{equation}
R_S(\bar{w}) = R_S\left(\frac{1}{\card{K_f}}
\sum_{k \in K_f} w_k\right)
\le \frac{1}{\card{K_f}}\sum_{k \in K_f} R_S(w_k)
\end{equation}
Further, for all $S \in \C$ and all $k \in [T]$,
by the assumed tolerance of $R_S$, we know that
$$\card{R_S(w_k) - \hat{R}_S(w_k)} \le \tau/5.$$
Given that for all $k \in K_f$, $\hat{R}_{S_k}(w_k) \le 4\tau/5$,
then applying the triangle inequality, we conclude that
for each comparison $S \in \C$,
\begin{align*}
\E_{(x,x') \sim S}\big[\card{f_{\bar{w}}(x) - f_{\bar{w}}(x')} -
d(x,x')\big] = R_S(\bar{w}) \le \tau.
\end{align*}
Hence, $\bar{w}$ is $(\C,d,\tau)$-metric multifair.
\end{proof}

\paragraph{Utility and runtime analysis}
We analyze the utility of Algorithm~\ref{alg:switching}
using a duality argument.  For notational convenience, denote
$L(w) = \E_{x_i \sim \D}[L(f_w(x_i),y_i)]$.
In addition to the assumptions in the main body,
throughout, we assume the following bounds on the
subgradients for all $w \in \F$.
\begin{align}
\forall S \in \C:&~\norm{\grad R_S(w)}_\infty \le m &\norm{\grad L(w)}_\infty \le g
\end{align}
Assuming an $\ell_\infty$ bound implies a bound on
the corresponding second moments of the stochastic subgradients;
specifically, we use the notation
$\norm{\grad R_S(w)}_2^2 \le M^2 = m^2n$ and
$\norm{\grad L(w)}_2^2 \le G^2 = g^2n$.

Consider the Lagrangian of the program
${\cal L}:\F \times \R_+^{\card{\C}} \to \R$.
\begin{equation}
{\cal L}(w,\lambda) = L(w) + \sum_{S \in \C} \lambda_S R_S(w)
\end{equation}
Let $w_* \in \F$ be an optimal feasible hypothesis;
that is, $w_*$ is a $(\C,d,0)$-metric multifair hypothesis
such that $L(w_*) \le L(w)$ for all other $(\C,d,0)$-metric multifair
hypotheses $w \in \F$.\footnote{Such a $w^*$ exists, as $w = 0 \in \R^n$
always trivially satisfies all the fairness constraints.}
By its optimality and feasibility, we know that $w_*$ achieves objective value
$L(w_*) = \inf_{w \in F}\sup_{\lambda \in \R_+^{\card{\C}}} {\cal L}(w,\lambda).$
Recall, the dual objective is given as
$D(\lambda) = \inf_{w \in \F}{\cal L}(w,\lambda)$.
Weak duality tells us that the dual objective value is
upper bounded by the primal objective value.
\begin{equation}
\sup_{\lambda \in \R_+^{\card{\C}}} D(\lambda)
\le L(w_*)
\end{equation}
As there is a feasible point and the convex constraints induce
a polytope, Slater's condition is satisfied and strong duality holds.
To analyze the utility of $\bar{w}$, we choose a setting of
dual multipliers $\bar\lambda \in \R_+^{\card{C}}$ such that
the duality gap $\gamma(w,\lambda) = L(w) - D(\lambda)$
is bounded (with high probability over the random choice of
stochastic subgradients).
Exhibiting such a setting of $\bar{\lambda}$
demonstrates the near optimality of $\bar{w}$.
\begin{lemma}
\label{lem:utility}
Let $\tau,\delta > 0$ and $\F = [-B,B]^n$.
After running Algorithm~\ref{alg:switching}
for $T > \frac{30^2M^2B^2n\log(n/\delta)}{\tau^2}$ iterations, then
with probability at least $1 - 8\delta$ (over the stochastic subgradients)
$$L(\bar{w}) \le L(w_*) + \frac{3G}{5M}\tau.$$
\end{lemma}
We give the full proof of Lemma~\ref{lem:utility} in Appendix~\ref{app:sec:proof}.

\subsection{Answering residual queries}
Next, we describe how to answer residual queries $R_S(w)$
efficiently, in terms of time and samples.
\begin{lemma}
\label{lem:residual}
For $\tau, \gamma >0$, for a $\gamma$-large collection of comparisons
$\C \subseteq 2^{\X\times \X}$, with probability $1-\delta$,
given access to $n$ metric samples,
every residual query $R_S(w)$ can be answered correctly with
tolerance $\tau$ provided
$$n \ge \tilde{\Omega}\left(\frac{\log(\card{C}/\delta)}{\gamma \cdot \tau^2}\right).$$
Each residual query $R_S(w)$ can be answered
after $\tilde{O}\left(\frac{\log(T\cdot \card{\C}/\delta)}{\gamma \cdot \tau^2}\right)$
evaluations of the current hypothesis.
\end{lemma}
\begin{proof}
Recall the definition of $R_{S}(w)$.
$$
R_{S}(w) = \E_{(x,x') \sim S}\big[\card{f_w(x) - f_w(x')}\big]
- \E_{(x,x') \sim S}\big[d(x,x')\big]
$$
Proposition~\ref{prop:metric:sample} shows that
$\E_S[d(x,x')]$ can be estimated for all $S \in \C$
from a small number of metric samples.  The proof follows a standard
Chernoff plus union bound argument.  For completeness, we give
a full proof next.
Thus, Lemma~\ref{lem:residual} follows by showing that at each iteration
$\E_{S}[\card{f_w(x) - f_w(x')}]$ can be estimated from a small
number of evaluations of the current hypothesis $f_w$.

We can estimate the expected value of the deviation on $f$ over $S \in \C$
with a small set of unlabeled samples from $\X \times \X$; we will evaluate
the hypothesis $f$ for each of these samples.
Using an identical argument as in the case of the expected metric value,
we can prove the following bound on how many comparisons we need
to make, which shows the lemma.
\begin{proposition}
\label{prop:metric:sample}
Suppose $\C$ is $\gamma$-large.  Then with probability at least
$1-\delta$, for all $S \in \C$, the empirical estimate for
$\E_S[\card{f(x) - f(x')}]$ of $n$ samples
$(x,x') \sim \M$ deviates from the
true expected value by at most $\tau$ provided
$$n \ge \tilde{\Omega}\left(\frac{B^2\log(\card{C}/\delta)}{\gamma \cdot \tau^2}\right).$$
\end{proposition}
\end{proof}

Here, we show that
a small number of samples from the metric suffices to estimate the
expected metric distance over all $S \in \C$.
Suppose $\C$ is $\gamma$-large.  Then with probability at least $1-\delta$,
for all $S \in \C$, the empirical estimates for $\E_{S}[d(x_i,x_j)]$
of $n$ metric samples deviate from their true expected value by at
most $\tau$ provided
$$n \ge \tilde{\Omega}\left(\frac{\log(\card{\C}/\delta)}{\gamma \cdot \tau^2}\right).$$

\begin{proof}
Let $(x_i,x_j,\Delta_{ij})$ represent a random metric sample.
Suppose for each $S \in \C$, we obtain $m$ such samples where
$(x_i,x_j) \sim S$,
and let $d(S) = \sum_{ij \in X_S} \Delta_{ij}$
be the empirical average over the sample.
Then, by Hoeffding's inequality, we know
$$\Pr\left[\card{d(S) - \E_{(x_i,x_j) \sim S}[d(x_i,x_j)]} > \tau \right]
\le 2e^{-2s\tau^2}.$$
If $m \ge \Omega\left(\frac{\log(\card{C}/\delta)}{\tau^2}\right)$, then
the probability that the estimate $d(S)$ is not within $\tau$ of its
true value is less than $\frac{\delta}{\card{C}}$.  Union bounding over
$\C$, the probability that every estimate has tolerance $\tau$ will be at least
$1-\delta$.

Because $\C$ is $\gamma$-large, for every $S \in \C$,
the probability a random metric sample $(x_i,x_j) \sim \M$
is in $S$ is at least $\gamma$.
If we take $\frac{\log(m)}{\gamma}$ samples,
then with probability at least $1-1/m$, one of the samples will be
in $S$.  Thus, to guarantee $d(S)$ has tolerance $\tau$
for all $S \in \C$ with probability $1-\delta$,
$$s = \frac{m\log(m)}{\gamma} =
\tilde{\Omega}\left(\frac{\log(\card{\C}/\delta)}{\gamma\cdot \tau^2}\right)$$
samples suffice.
\end{proof}

\subsection{Answering subgradient queries}
Next, we argue that the subgradient oracles can be implemented
efficiently without accessing any metric samples.
First, suppose we want to take a step according to $R_S(w)$;
while $R_S(w)$ is not differentiable, we can compute a legal
subgradient defined by partial subderivatives given as follows.
\begin{equation}
\frac{\partial R_S(w)}{\partial w_l} =
\E_{(x,x') \sim S}[\sgn(\langle w,x-x' \rangle) \cdot (x_{l} - x'_{l})]
\end{equation}
The subgradient does not depend on $d$, so no samples from the
metric are necessary.
Further, Algorithm~\ref{alg:switching}
only assumes access to stochastic subgradient oracle with bounded
entries.  If we sample a single $(x_i,x_j) \sim \M$,
then $\sgn(\langle w,x_i-x_j \rangle) \cdot (x_{il} - x_{jl})$ will
be an unbiased estimate of a subgradient of $R_S(w)$; we claim, the
entries will also be bounded.
In particular, assuming $\norm{x_i}_1 \le 1$ implies each partial
is bounded by $2$, so that we can take $M^2 = 4n$.

\subsection{Utility analysis of Algorithm~\ref{alg:switching}}
\label{app:sec:proof}

In this appendix, we give a full proof of Lemma~\ref{lem:utility}.
We defer the proof of certain technical lemmas to Appendix~\ref{app:sec:algorithm}
for the sake of presentation.

\paragraph{Proof of Lemma~\ref{lem:utility}}
Let $\tau,\delta > 0$ and $\F = \set{w \in \R^n : \norm{w}_\infty \le B}$.
After running Algorithm~\ref{alg:switching}
for $T > \frac{30^2M^2B^2n\log(n/\delta)}{\tau^2}$ iterations, then
$$L(\bar{w}) \le L(w_*) + \frac{3G}{5M}\tau$$
with probability at least $1 - 8\delta$ over the randomness of the algorithm.
\begin{proof}
As before, we refer to $K_f \subseteq [T]$ as the set of feasible iterations,
where we step according to the objective,
and $[T] \setminus K_f$ as the set of infeasible iterations,
where we step according to the violated constraints.
Recall, we denote the set of subgradients of a
function $L$ (or $R$) at $w$ by $\partial L(w)$
and denote by $\grad L(w)$ a stochastic subgradient,
where $\E[\grad L(w) \given w] \in \partial L(w)$.

When we do not step according
to the objective, we step according to the subgradient of some
violated comparison constraint.  In fact,
we show that stepping according to any convex combination
of such subgradients suffices to guarantee progress in the
duality gap.
In the case wher $t \not \in K_f$,
we assume that we can find some convex combination
$\sum_{S \in \C} \alpha_{k,S} \hat{R}_S(w_k) > 4\tau/5$
where for all $S \in \C$, $\alpha_{k,S} \in \Delta_{\card{\C}-1}$.
We show that if we step
according to the corresponding combination of the
subgradients of $R_S(w_k)$, we can bound the duality gap. Specifically,
for $k \not \in K_f$, let the algorithm's step be given by
$$\sum_{S \in \C}\alpha_{k,S}\grad R_S(w_k)$$ where for each $S \in \C$, we have
$\E\left[\grad R_S(w_k) \given w_k\right] \in \partial R_S(w_k)$.
Let $\eta_L = \frac{\tau}{GM}$ and $\eta_R = \frac{\tau}{M^2}$
denote the step size for the objective and residual steps,
respectively. Then, consider the following choice of dual multipliers
for each $S \in \C$.
\begin{equation}
\bar{\lambda}_S = \frac{\eta_R}{\eta_L \card{K_f}}\sum_{k \not \in K_f}\alpha_{k,S}
\end{equation}
Expanding the definition of $\bar{w}$ and applying convexity,
we can bound the duality gap as follows
\begin{align}
\gamma(\bar{w},\bar{\lambda}) & = L(\bar{w}) - D(\bar{\lambda})\\
&\le \frac{1}{\card{K_f}}\left(\sum_{k \in K_f}
L(w_k)\right) - \inf_{w \in \F}\set{L(w) + \sum_{S \in \C}
\bar{\lambda}_S R_S(w)} \label{ineq:gap:defns}
\\
&= \sup_{w \in \F}\set{
\frac{1}{\card{K_f}}\left(\sum_{k \in K_f}
L(w_k)\right) - L(w) - \sum_{S \in \C} \bar{\lambda}_S R_S(w)
}\\
&= \sup_{w \in \F}\set{
\frac{1}{\eta_L \card{K_f}}\left(\eta_L\sum_{k \in K_f}
(L(w_k) - L(w)) - \eta_R\sum_{k \not \in K_f}\sum_{S \in \C}
\alpha_{k,S}R_S(w) \right)} \label{ineq:gap:convex}
\end{align}
where (\ref{ineq:gap:defns}) follows from expanding $\bar{w}$ then applying convexity of $L$
and the definition of $d(\bar{\lambda})$ and (\ref{ineq:gap:convex})
follows by our choice of $\bar\lambda_S$ for each $S \in \C$.

With the duality gap expanded into one sum over the feasible iterates
and one sum over the infeasible iterates, we can analyze these iterates
separately.  The following lemmas show how to track
the contribution of each term to the duality gap in
terms of a potential function $u_k$ defined as
$$u_{k}(w) = \frac{1}{2}\norm{w-w_k}^2.$$
For notational convenience, for each $k \in K_f$,
let $e(w_k) = \E[\grad L(w_k) \given w_k] - \grad L(w_k)$
be the noise in the subgradient computation.
\begin{lemma}
\label{lem:progress:feasible}
For all $w \in \F$ and for all $k \in K_f$,
$$\eta_L \cdot \left(L(w_k) - L(w)\right) \le 
u_k(w) - u_{k+1}(w) + \frac{\tau^2}{2M^2} 
+ \eta_L \langle e(w_k), w_k - w \rangle.$$
\end{lemma}
Again, for notational convenience, for each $k \in [T]\setminus K_f$,
let $e(w_k) = \sum_{S \in \C} \alpha_{k,S}
\left(\E[\grad R_{S}(w_k) \given w_k] - \grad R_{S}(w_k)\right)$
be the noise in the subgradient computation.
\begin{lemma}
\label{lem:progress:infeasible}
For all $w \in \F$ and for all $k \in [T] \setminus K_f$,
$$-\eta_R \sum_{S \in \C} \alpha_{k,S}R_{S}(w) \le u_k(w) - u_{k+1}(w) - \frac{\tau^2}{10M^2}
+ \eta_R \langle e(w_k), w_k - w \rangle.$$
\end{lemma}
We defer the proofs of Lemmas~\ref{lem:progress:feasible}
and \ref{lem:progress:infeasible} to Appendix~\ref{app:sec:algorithm}.
Assuming Lemmas~\ref{lem:progress:feasible} and \ref{lem:progress:infeasible},
we bound the duality gap as follows.
\begin{multline}
\sup_{w \in \F}\left\{\frac{1}{\eta_L\card{K_f}}\left(
\sum_{k=1}^T\big[u_{k-1}(w) - u_k(w) \big]
+ \eta_L\sum_{k \in K_f} \langle e(w_k),w_k - w \rangle
\right.\right.\\ \left.\left.\vphantom{\frac{1}{\Lambda} \sum_{k=1}^T\big[u_{k-1}(w) - u_k(w) \big]}
+ \eta_R\sum_{k \not \in K_f} \langle e(w_k),w_k - w \rangle
+ \frac{\tau^2}{2M^2}\card{K_f}
- \frac{\tau^2}{10 M^2}(T - \card{K_f})\right)\right\}
\end{multline}
\begin{multline}
\le \frac{1}{\eta_L\card{K_f}}
\underbrace{\left(\sup_{w \in \F} \set{u_0(w)
+ \eta_L\sum_{k \in K_f} \langle e(w_k), w_k - w \rangle
+ \eta_R\sum_{k \not \in K_f} \langle e(w_k), w_k - w \rangle }
- \frac{\tau^2}{10M^2}T\right)}_{(*)}
\\
+ \underbrace{\frac{G\tau}{2M} + \frac{G\tau}{10 M}}_{(**)}\label{ineq:gap:continued}
\end{multline}
by rearranging.
Noting that $(**)$ can be bounded by $\frac{3G}{5M}\tau$,
it remains to bound $(*)$.
We show that for a sufficiently large $T$, then $(*)$
cannot be positive.

Consider the terms in the supremum over $w \in \F$.
Note that we can upper bound $\sup\set{u_0(w)} \le 2B^2n$.
Additionally, we upper bound the error incurred due to
the objective subgradient noise with the following lemma,
which we prove in Appendix~\ref{app:sec:algorithm}.
\begin{lemma}
\label{lem:noisy}
With probability at least $1-4\delta$, the contribution of the noisy
subgradient computation to the duality gap can be bounded as follows.
\begin{equation}
\sup_{w \in \F}\set{\eta_L \sum_{k \in K_f}\langle e(w_k),w_k-w \rangle
+ \eta_R \sum_{k \not \in K_f}\langle e(w_k), w_k - w \rangle}
\le \frac{\tau B}{M}\sqrt{8Tn\log(n/\delta)}
\end{equation}
\end{lemma}
Thus, we can bound $(*)$ as follows.
$$ (*) \le 2B^2n + \frac{\tau B}{M}\sqrt{8Tn\log(n/\delta)} - \frac{\tau^2}{10M^2}T$$
Assuming the lemma and that
$T > \frac{30^2M^2 B^2 n\log(n/\delta)}{\tau^2}$, then,
we can bound $(*)$ by splitting the negative
term involving $T$ to balance both positive
terms.
\begin{equation}
(*) \le \left(2B^2n - \frac{\tau^2}{10M^2}\cdot \frac{20T}{30^2}\right)
+ \left(\frac{\tau B}{M}\sqrt{8n\log(n/\delta)}
\cdot \sqrt{T} - \frac{\tau^2}{10M^2}\cdot \frac{(30^2-20) T}{30^2}\right)
\end{equation}
\begin{multline}
\le \left(2B^2n - \frac{\tau^2}{10M^2}
\frac{20M^2B^2n\log(n/\delta)}{\tau^2}\right)\\ +
\left(\frac{\tau B}{M}\sqrt{8n\log(n/\delta)}\cdot
\frac{30MB\sqrt{n\log(n/\delta)}}{\tau}
- \frac{\tau^2}{10M^2}\cdot \frac{(30^2-20) M^2 B^2n\log(n/\delta)}{\tau^2}
\right)
\end{multline}
\begin{equation}
\le \left(2B^2n - 2B^2n\log(n/\delta)\right)\\ +
\left(85B^2n\log(n/\delta) - 88 B^2n\log(n/\delta)\right)
\end{equation}
Thus, the sum of $(*)$ and $(**)$ is at most $\frac{3G}{5M}\tau$.
\end{proof}

\subsection{Deferred proofs from analysis of Algorithm~\ref{alg:switching}}
\label{app:sec:algorithm}

\paragraph{Technical lemma}
First, we show a technical lemma that will be useful in
analyzing the iterates' contributions to the duality gap.
Recall our potential function $u_k:\F \to \R$.
\begin{equation}
u_k(w) = \frac{1}{2}\norm{w_k - w}_2^2
\end{equation}
We show that the update rule $w_{k+1} \gets \pi_{\F}(w_k - \eta_k g_k)$ implies
the following inequality in terms of $\eta_k,g_k,u_k(w)$, and $u_{k+1}(w)$.
\begin{lemma}
\label{lem:potential}
Suppose $w_{k+1} = \pi_{\F}(w_k - \eta_k g_k)$.  Then, for all $w \in \F$,
\begin{equation}
\eta_k \langle g_k, w_{k} - w \rangle \le u_k(w) - u_{k+1}(w) +\frac{\eta_k^2}{2}\norm{g_k}_2^2.
\end{equation}
\end{lemma}
\begin{proof}
Consider the differentiable, convex function $B_k:\F \to \R$.
\begin{equation}
B_{k}(w) = \eta_k \langle g_k, w-w_k \rangle + \frac{1}{2}\norm{w-w_k}_2^2
\end{equation}

\begin{align}
\langle \grad B_k(w_{k+1}), w- w_{k+1} \rangle
&= \langle \eta_k g_k + w_{k+1} - w_k, w - w_{k+1} \rangle
\label{ineq:first-order:expand}\\
&= \langle \pi_{\F}(w_k - \eta_k g_k) - (w_k - \eta_k g_k),
w - \pi_{\F}(w_k - \eta_k g_k) \rangle \label{ineq:first-order:project}\\
& \ge 0 \label{ineq:first-order:0}
\end{align}
where (\ref{ineq:first-order:project}) follows by substituting
the definition of $w_{k+1}$ twice;
and (\ref{ineq:first-order:0})
follows from the fact that for any closed convex set $\F$ and
$w_0 \not \in \F$,
$$\langle \pi_{\F}(w_0) - w_0, w - \pi_{\F}(w_0) \rangle \ge 0.$$

Rearranging (\ref{ineq:first-order:expand}) implies the
following inequality holds for all $w \in \F$.
\begin{align}
& \langle \eta_k g_k + w_{k+1} - w_k, w - w_{k+1} \rangle \ge 0\\
\iff & \eta_k \langle g_k, w_{k+1} - w \rangle \le
\langle w_{k+1} - w_k, w- w_{k+1} \rangle \label{ineq:ws}
\end{align}
We will use the following technical identity to prove the lemma.
\begin{proposition}\label{fact:identity}
For all $w \in \F$,
$$
\langle w_{k+1} - w_k, w - w_{k+1} \rangle =
u_k(w) - u_{k+1}(w) - \frac{1}{2}\norm{w_{k+1} - w_k}^2.
$$
\end{proposition}
\begin{proof}
\begin{align*}
u_k(w) - u_{k+1}(w) &= \norm{w_k - w}^2 - \norm{w_{k+1} - w}^2\\
&= \norm{w_k}^2 + \norm{w}^2 - \norm{w_{k+1}}^2 - \norm{w}^2 +
2\langle w_{k+1} - w_k,w \rangle\\
&= \norm{w_k}^2 - \norm{w_{k+1}}^2+
2\langle w_{k+1} - w_k,w \rangle\\
&= \norm{w_k}^2 - \norm{w_{k+1}}^2
- 2\langle w_k - w_{k+1},w_{k+1}\rangle
+ 2\langle w_{k+1} - w_k,w - w_{k+1}\rangle\\
&=\norm{w_{k+1} - w_k}^2
+ 2\langle w_{k+1} - w_k, w - w_{k+1} \rangle
\end{align*}
\end{proof}
Finally, we can show the inequality stated in the lemma.
\begin{align}
\eta_k \langle g_k, w_k - w \rangle &=
\eta_k \langle g_k, (w_{k+1} + \eta g_k) - w \rangle\\
&\le \langle w_{k+1} - w_k, w- w_{k+1}\rangle + \eta_k^2 \norm{g_k}_2^2 \label{ineq:lem:ws} \\
&\le u_k(w) - u_{k+1}(w) - \frac{1}{2}\norm{w_{k+1}-w_k}_2^2 + \eta_k^2 \norm{g_k}_2^2 \label{ineq:lem:us} \\
&= u_k(w) - u_{k+1}(w) + \frac{\eta_k^2}{2}\norm{g_k}_2^2
\label{ineq:lem:step}
\end{align}
where (\ref{ineq:lem:ws}) follows from (\ref{ineq:ws});
(\ref{ineq:lem:us}) follows by
using Proposition~\ref{fact:identity} to write the expression in terms of $u_k$'s; and (\ref{ineq:lem:step}) follows by
the gradient step $w_{k+1}-w_k = \eta_k g_k$.
\end{proof}

\paragraph{Proof of Lemma~\ref{lem:progress:feasible}}
Here, we bound the contribution to the duality gap of each of the
feasible iterations $k \in K_f$ as follows.
\begin{equation}
\eta_L \cdot \left(L(w_k) - L(w)\right) \le 
u_k(w) - u_{k+1}(w) + \frac{\tau^2}{2M^2} + \eta_L\langle e_L(w_k), w_k - w \rangle
\end{equation}
\begin{proof}
Let $e_L(w_k) = g_L(w_k) - \grad L(w_k)$ where
$g_L(w_k) = \E[\grad L(w_k)] \in \partial L(w_k)$.
\begin{align}
\eta_L \cdot \left(L(w_k) - L(w)\right)
&\le \eta \langle g_L(w_k), w_k - w \rangle \\
&\le \eta \langle \grad L(w_k) + e_{L}(w_k), w_k - w \rangle
\label{ineq:feasible:error}\\
&\le u_k(w) - u_{k+1}(w) + \frac{\eta_L^2}{2}\norm{\grad L(w_k)}_2^2
+ \eta_L\langle e_L(w_k), w_k - w \rangle \label{ineq:feasible:potential}\\
&\le u_k(w) - u_{k+1}(w) + \frac{\tau^2}{2M^2} + \eta_L\langle e_L(w_k), w_k - w \rangle \label{ineq:feasible:eta}
\end{align}
where (\ref{ineq:feasible:error}) follows by substituting $g_L$;
(\ref{ineq:feasible:potential})
follows by expanding the inner product and applying Lemma~\ref{lem:potential}
to the first term; (\ref{ineq:feasible:eta}) follows by our choice of
$\eta_L = \tau/GM$.
\end{proof}

\paragraph{Proof of Lemma~\ref{lem:progress:infeasible}}
Here, we bound the contribution to the duality gap of each of the
infeasible iterates $k \in [T] \setminus K_f$.  We assume
$\hat{R}_{S_k}(w_k)$ has tolerance $\tau/5$.  Then we show
\begin{equation}
- \eta_R \sum_{S \in \C}\alpha_{k,S} R_{S}(w) \le u_k(w) - u_{k+1}(w) - \frac{\tau^2}{10M^2}
+ \eta_R\langle e_R(w_k), w_k - w \rangle.
\end{equation}
\begin{proof}
Recall, we let
$e(w_k) = \sum_{S \in \C} \alpha_{k,S}
\left(\E[\grad R_{S}(w_k) \given w_k] - \grad R_{S}(w_k)\right)$.
For each $S \in \C$, for any $g_S(w_k) \in \partial R_S(w_k)$,
we can rewrite $-R_{S}(w)$ as follows.
\begin{align*}
-R_{S}(w) &= R_{S}(w_k) - R_{S}(w) - R_{S}(w_k)\\
&\le \langle g_S(w_k), w_k - w \rangle - R_{S}(w_k)
\end{align*}
Multiplying by $\eta_R$ and taking the convex combination of $S \in \C$ according to $\alpha_k$,
we apply Lemma~\ref{lem:potential} to obtain
the following inequality.
\begin{align}
- \eta_R \sum_{S \in \C}\alpha_{k,S}R_{S}(w) &\le 
\eta_R \left\langle \sum_{S \in \C}\alpha_{k,S}\grad R_{S}(w_k)
+ e_R(w_k), w_k - w \right\rangle
- \eta_R\sum_{S \in \C}\alpha_{k,S} R_{S}(w_k) \label{ineq:infeasible:err}\\
\begin{split}
&\le u_k(w) - u_{k+1}(w) + \frac{\eta_R^2}{2}\norm{\sum_{S \in \C}\alpha_{k,S}\grad R_{S}(w_k)}_2^2\\
&\hspace{1.4in}- \eta_R \sum_{S \in \C}\alpha_{k,S}R_{S}(w_k) + \eta_R\langle e_R(w_k), w_k - w \rangle
\label{ineq:infeasible:potential}
\end{split}\\
&\le u_k(w) - u_{k+1}(w) + \frac{\tau^2}{2 M^2}
- \frac{\tau}{M^2}\cdot \sum_{S \in \C}
\alpha_{k,S} R_{S}(w_k)
+ \eta_R\langle e_R(w_k), w_k - w \rangle \label{ineq:infeasible:R}\\
&\le u_k(w) - u_{k+1}(w) - \frac{\tau^2}{10M^2}
+ \eta_R\langle e_R(w_k), w_k - w \rangle
\label{ineq:infeasible:tau}
\end{align}
where (\ref{ineq:infeasible:err}) follows by substituting $\grad R_S(w_k)$ for each $g_S(w_k)$
and the definition of $e_R(w_k)$;
(\ref{ineq:infeasible:potential}) follows by expanding the inner product
and applying Lemma~\ref{lem:potential}; (\ref{ineq:infeasible:R})
follows by our choice of
$\eta_k = \tau^2/M^2$;
(\ref{ineq:infeasible:tau}) follows by the fact that when we update
according to a constraint, we know
$\sum_{S \in \C}\alpha_{k,S}\hat{R}_{S}(w_k) \ge 4\tau/5$
with tolerance $\tau/5$, so $\sum_{S \in \C}
\alpha_{k,S}R_{S}(w_k) \ge 3\tau/5$.
\end{proof}

\paragraph{Proof of Lemma~\ref{lem:noisy}}
Here, we show that
with probability at least $1-4\delta$, the contribution of the noisy
subgradient computation to the duality gap can be bounded as follows.
\begin{equation}
\sup_{w \in \F}\set{\eta_L \sum_{k \in K_f}\langle e(w_k),w_k-w \rangle
+ \eta_R \sum_{k \not \in K_f}\langle e(w_k), w_k - w \rangle}
\le \frac{\tau B}{M}\sqrt{8Tn\log(n/\delta)}
\end{equation}

\begin{proof}
Let $\eps = \eta_L\cdot g = \eta_R \cdot m =
\frac{\tau}{M\sqrt{n}}$.  Further, let $\eta_k = \eta_L$
for $k \in K_f$ and $\eta_R$ for $k \not \in K_f$.
Then, we can rewrite the expression to bound using
$\eta_k$ and expand as follows.
\begin{align*}
\sup_{w \in \F} \sum_{k \in [T]} \langle \eta_k e(w_k), w_k - w \rangle
&= \sup_{w \in \F} \sum_{k\in [T]} \sum_{l = 1}^n
\eta_k e(w_k)_l \cdot (w_k-w)_l\\
&= \sum_{l = 1}^n (w_{k})_l \sum_{k \in [T]} \eta_k e(w_k)_l
+ \sum_{l=1}^n\sup_{(w)_l} \set{(w)_{l} \cdot \sum_{k\in [T]} \eta_k e(w_k)_l}
\end{align*}
Consider the second summation, and consider the summation inside
the supremum.  Note that this summation is a sum of mean-zero random
variables, so it is also mean-zero.  Recall, we assume the estimate
of the $k$th subgradient is independent of the prior subgradients,
given $w_k$. Further, by the assumed $\ell_\infty$ bound on the
subgradients, each of these random variables is bounded in magnitude
by $\eps$.  Using the bounded difference property,
we apply Azuma's inequality separately for each $l \in [n]$.
\begin{align*}
\Pr\left[\card{\sum_{k\in [T]} \eta_k e(w_k)_l} > Z\right]
&\le 2 \cdot \exp\left(-\frac{Z^2}{2T\eps^2}\right)
\end{align*}
Taking this probability to be at most $2\delta/n$, we can upper bound
$Z$ by $\eps \sqrt{2T\log(n/\delta)} = \frac{\tau}{M}\sqrt{2T n \log(n/\delta)}$.
Then, noting that
$\card{(w)_l} < B$ for any $w \in \F$, we can
take a union bound to conclude with probability at least $1-2\delta$
the following inequalities hold.
\begin{align*}
\sum_{l = 1}^n \sup_{(w)_l} \set{(w)_l \cdot \sum_{k\in [T]} \eta_k e(w_k)_l}
&\le B n \cdot Z\\
&= \frac{\tau B}{M}  \sqrt{2T n\log(n/\delta)}
\end{align*}
Further, we note that the first summation concentrates
at least as quickly as the second, so by union bounding again,
$$ \sup_{w \in \F} \sum_{k \in [T]} \eta_k \langle e(w_k), w_k - w \rangle
\le \frac{\tau B}{M}  \sqrt{8Tn\log(n/\delta)}$$
with probability at least $1-4\delta$.
\end{proof}

\section{Hardness for metric multifair predictions}

The lower bounds follow the same general construction.
Suppose we have some boolean concept class
$\C \subseteq \set{-1,1}^{\X}$ for some universe $\X$.
We will construct a new universe $\X_{01} = \X_0 \cup \X_1$
where $\X_0 \subseteq \X$
and define a collection of ``bipartite'' comparisons over subsets of
${\X_0 \times \X_1}$. We will assume access to random
samples $(x_0,y(x_0))$ for some function $y:\X \to \set{-1,1}$;
we define corresponding metric values
where $d(x_0,x_1)$ is some function of $y(x_0)$ for
all $x_1 \in \X_1$.  Finally, we need to label
$\X_{01}$ such that such that to obtain good loss,
the post-processing algorithm must learn something
non-trivial about $y$ (which will differ across the two
lower bounds we prove).

Consider $\X_{01} = \X_0 \cup \X_1$ with
ideal labels given as $(x_0,0)$ for $x_0 \in \X_0$
and $(x_1,1)$ for $x_1 \in \X_1$, and let $L$ be
the hinge loss.  We encode
the original boolean concept in the distance metric,
where for $x_0 \in \X_0$,
$$d(x_0,x_1) = 1 - y(x_0)$$ for all $x_1 \in \X_1$.

Then, consider the collection of comparisons given
by $\H = \set{S_c : c \in \C}$ where we take
$S_c = \set{(x_0,x_1) \in \X_0 \times \X_1 : c(x_0) = 1}$.
We take $\card{\X_1} = 2\card{\X_0}$, large enough that the average
prediction for $x_1 \in \X_1$ is at least $1-\tau$ in the optimal
utility set of multifair predictions. In particular,
any average deviation by more than
$\tau$ in $\X_1$ would result in larger loss than
setting all of $f(x_0) = f(x_1) = 1-\tau$
for $x_1 \in \X_1$ and all $x_0 \in \X_0$ where $y(x_0) = 1$.

\paragraph{Outline of sample complexity lower bound.}
Here, we outline the proof of Theorem~\ref{thm:lb}
\begin{theorem*}[Restatement of Theorem~\ref{thm:lb}]
Let $\gamma,\tau > 0$ be constants and
suppose $\A$ is an algorithm that
has random sample access to $d$ and outputs
a $(\C,d,\tau)$-metric multifair set of predictions
for $\gamma$-large $\C$.
Then, $\A$ takes $\Omega(\log\card{\C})$
random samples from $d$ or outputs a set of predictions
with loss that approaches the loss achievable with no metric
queries.
\end{theorem*}
\begin{proof}
The theorem follows from our construction.
In particular, if we take the
concept class $\C$ on $\X_0$ to be set of
linear functions over $\mathbb{F}_2$, and take
$y = c$ to be a uniformly random $c \in \C$,
then we get a hard distribution over metrics.
Specifically, if the we take the concepts to
be $n$-dimensional, then without $n$ linearly-independent
queries to the metric,
we will not be able to learn the concept
with non-trivial accuracy.  In particular,
any algorithm that guarantees metric multifairness
must assume that $\E_{S}[d(x_0,x_1)] \approx 0$
more than one $S \in \H$, which results in suboptimality.
Thus, the assumption that the metric multifair
learner achieved near-optimal utility must be false.
\end{proof}
Further note, if we only take $n-k$ queries, the incurred loss
approaches the trivial loss exponentially in $k$.
Appealing to lower bounds on the number of random queries
needed to span a basis, the theorem follows.

\paragraph{Outline of hardness from pseudorandom functions}
We use the construction to demonstrate that
under weak complexity assumptions,
there are algorithmic barriers to generally efficient
algorithms for metric multifairness.  In particular,
we will assume that $\C$ defines a pseudorandom function
family.  The existence of one-way functions implies
the existence of pseudorandom functions \cite{ggm}, so the
proposition follows.

\begin{proposition*}[Formal statement of Proposition~\ref{prop:lb}]
Assuming one-way functions exist, any
algorithm for computing $(\H,\tau)$-optimal
$(\H,d,\tau)$-metric multifair predictions
for any $\H,d,$ and $\tau > 0$
requires time $(\log\card{\H})^{\omega(1)}$.
\end{proposition*}

\begin{proof}
Suppose we can post-process predictions to achieve
$(\C,\tau)$-optimal
$(\C,d,\tau)$-metric multifair predictions
for any $\C$ and $d$ and some small constant
$\tau > 0$.
Let $\C$ define a pseudorandom function family.
We will show that we can distinguish between
a function $y = c \gets_R \C$ drawn uniformly at random
from the function family and a truly random
function $y:\X \to \set{-1,1}$.

We're given some samples of the form
$(x,y(x)) \sim \D \times \set{-1,1}$.
Returning to the proposed construction,
in the case $y$ is a truly random function,
then with high probability,
$\E_S[d(x_0,x_1)] \ge 1-o(1)$ for all
$S \in \H$.  Thus, any $\tau$-optimal
set of predictions will achieve loss $O(\tau)$.

In the case, where $y = c$ for some $c \in \C$,
note that labeling $x_0 \in \X_0$ according to
$c(x_0)$ and all $x_1 \in \X_1$
as $1$ is a feasible
point that obtains loss
\begin{align*}
\E_{x \sim \X_0}[L(f(x),0)]/3
+ 2\cdot \E_{x \sim \X_1}[L(f(x),1)]/3
&= \Pr_{x_0 \sim \X_0}[c(x_0) = 1]/3
\end{align*}
This feasible quantity upper bounds the optimal loss.
Then, we can express the expectation of difference in the predictions
across the set defined by $c$ as follows.
\begin{align*}
\E_{(x_0,x_1) \sim S_c}\big[\card{f(x_0) - f(x_1)}
- (1-c(x_0))\big] &=
\E_{(x_0,x_1) \sim S_c}\big[\card{f(x_0) - f(x_1)}\big]
- \E_{(x_0,x_1) \sim S_c}\big[(1-c(x_0))\big]\\
&\ge \E_{x_0: c(x_0) = 1}[\card{f(x_0) - f(x_1)}] + 0.
\end{align*}
We assume the set of predictions $f$ is $(\H,d,\tau)$-metric multifair.
Thus, because $S_c \in \H$, we can upper bound this term by $\tau$.
In total then,
$\E_{x_0: c(x_0) = 1}[\card{f(x_0) - f(x_1)}] \le \tau$, and
by the fact that $\E_{x_1 \sim \X_1}[f(x_1)] \ge 1-\tau$,
then $\E_{x_0: c(x_0) = 1}[f(x_0)] \ge 1-2\tau$.
Further, consider the following lower bound on the loss on $f$.
\begin{align*}
3\cdot&\E_{(x,y) \sim \X\times[-1,1]}\left[~\max\set{0,\card{f(x) - y}}~\right]\\
&\ge \E_{x_0 \sim \D}\left[ f(x_0) \right]\\
&\ge \Pr_{x_0 \sim \D}[c(x_0) = 1]\cdot (1-2\tau)\\
&\ge \Pr_{x_0 \sim \X_0}[c(x_0) = 1] - 2\tau
\end{align*}
If $\C$ is a pseudorandom function
family the $\Pr_{x_0 \sim \X_0}[c(x_0)=1]
\approx 1/2$
must be bounded away from $0$.
Thus, in the case where $y \in \C$,
we have a non-trivial lower bound on the
achievable loss under
$(\C,d,\tau)$-metric multifairness.
Thus, we can distinguish when $y \in \C$
and when $y$ is truly random.
\end{proof}

\end{document}